\documentclass[letterpaper, 10 pt, conference]{ieeeconf}
\pdfoutput=1

\IEEEoverridecommandlockouts                              
\usepackage{cite}
\usepackage{amsmath,amssymb,amsfonts}
\usepackage{algorithmic}
\usepackage{graphicx}
\usepackage{textcomp}
\usepackage{xcolor}
\def\BibTeX{{\rm B\kern-.05em{\sc i\kern-.025em b}\kern-.08em
  T\kern-.1667em\lower.7ex\hbox{E}\kern-.125emX}}
\usepackage{wrapfig}
\usepackage{times}
\usepackage{tikz}
\usepackage{mathtools}
\usepackage{url}
\usetikzlibrary{automata, positioning}

\newtheorem{definition}{Definition}
\newtheorem{theorem}{Theorem}
\newtheorem{remark}{Remark}

\newtheorem{lemma}{Lemma}

\newtheorem{corollary}{Corollary}

\definecolor{darkgreen}{RGB}{0,127,0}

\DeclarePairedDelimiter\floor{\lfloor}{\rfloor}

\title{\LARGE \bf
Convergence Guarantee of Dynamic Programming for LTL Surrogate Reward
}

\author{Zetong Xuan and Yu Wang
\thanks{Zetong Xuan is with Department of Mechanical \& Aerospace Engineering,
        University of Florida, Gainesville, FL 32611, USA}
\thanks{Yu Wang is with Department of Mechanical \& Aerospace Engineering,
        University of Florida, Gainesville, FL 32611, USA
}
}
\begin{document}

\maketitle
\thispagestyle{empty}
\pagestyle{empty}

\begin{abstract}
Linear Temporal Logic (LTL) is a formal way of specifying complex objectives for planning problems modeled as Markov Decision Processes (MDPs). 
The planning problem aims to find the optimal policy that maximizes the satisfaction probability of the LTL objective. One way to solve the planning problem is to use the surrogate reward with two discount factors and dynamic programming, which bypasses the graph analysis used in traditional model-checking. 
The surrogate reward is designed such that its value function represents the satisfaction probability. 
However, in some cases where one of the discount factors is set to $1$ for higher accuracy, the computation of the value function using dynamic programming is not guaranteed. 
This work shows that a multi-step contraction always exists during dynamic programming updates, guaranteeing that the approximate value function will converge exponentially to the true value function. Thus, the computation of satisfaction probability is guaranteed. 
\end{abstract}


\section{Introduction}
\label{sec:intro}

Modern autonomous systems need to solve planning problems for complex rule-based tasks that
are usually expressible by linear temporal logic (LTL)\cite{pnueli1977temporala}. 
LTL is a symbolic language to describe high-level tasks like reaching a sequence of goals or ordering a set of events. 
The planning problem with the LTL objective is to find the optimal policy that maximizes the probability of satisfying the given LTL objective. 
This problem can be formally treated as a quantitative model-checking problem \cite{baier2008principles} when the environment is modeled as MDPs with known transition probability. That is, given an MDP and an LTL objective, find the maximum satisfaction probability within all possible policies. 

Although an LTL objective can be complex, its maximal satisfaction probability and optimal policy can be computed by quantitative (or probabilistic) model checking \cite{baier2008principles,fainekos2005temporal,kressgazit2009temporallogicbased} via reachability. 
First, a specific set of states is selected so the reachability probability to this set equals the maximum satisfaction probability of the LTL objective. 
The set of states shall be identified by graph analysis, such as a depth-first search, on the product MDP. 
The product MDP is a product of the original MDP and an $\omega$-regular automaton, encoding all the information necessary for quantitative model-checking. 
Then, dynamic programming or linear programming shall be applied to compute the reachability probability of the set. 

An alternative approach finds the optimal policy via a surrogate reward function without the graph analysis, thus is more generalizable to model-free reinforcement learning (RL) \cite{sadigh2014learning, hasanbeig2019reinforcement, hahn2020faithful, bozkurt2020control}. The surrogate reward is a reward function automatically derived from the LTL objective, which yields a value function representing the satisfaction probability of the LTL objective. The satisfaction probability can be computed using dynamic programming. 
Dynamic programming iteratively updates an approximate value function by the Bellman equation. 
Ideally, the approximate value function will converge to the value function during updates. 
Furthermore, the finding of the optimal policy and the corresponding maximum satisfaction probability can be done by policy or value iteration. 
Another advantage of using surrogate reward can be a smooth transformation into model-free RL to deal with situations when the MDP transitions are unknown. In this case, the two-phase model-checking approach has to be transformed into model-based RL and requires additional computation \cite{ashok2019paca}.

This work focuses on a widely used surrogate reward \cite{bozkurt2020control} based on the limit-deterministic B\"uchi automaton. 
It assigns a constant reward for ``good'' states with two discount factors. 
Given a policy, the probability of infinitely often visiting the ``good'' states shall equal the satisfaction probability of the LTL objective. 
The surrogate reward yields a value function equal to the satisfaction probability when taking the limit of both discount factors to one. 

Nevertheless, whether dynamic programming based on this surrogate reward can find the satisfaction probability has still not been fully studied. 
We noticed that more recent works \cite{voloshin2023eventual,shao2023sample,cai2021modular,hasanbeig2023certified} allow one discount factor to equal $1$ while using the surrogate reward with two discount factors from \cite{bozkurt2020control}. 
Their Bellman equations yield multiple solutions, as the discount factor of $1$ holds in many states.
The approximate value function updated by these Bellman equations may not converge to the value function. 
\cite{xuan2024uniqueness} proposes a sufficient condition to identify the value function from multiple solutions satisfying the Bellman equation.
However, when discounting is missing in many states, whether using dynamic programming or RL based on the Bellman equation shall give us the value function is still not fully answered. 

This work gives a convergence guarantee for the approximate value function during dynamic programming updates. 
We find an upper bound that decays exponentially for the infinite norm of the approximation error. 
Even though the one-step dynamic programming update does not provide a contraction on the approximation error, 
we show that a multi-step contraction always happens as we keep doing the dynamic programming update for enough steps.
We verify our result in the case study where our upper bound holds and approximation error goes to zero. 
The intuition behind our result is that although only a few states have discounting, the discount works on them shall be propagated to other states via the reachability between states. 

\section{Preliminaries} \label{sec:prelim}

This section introduces how an LTL objective can be translated to a surrogate reward function. 
First, we formulate the LTL planning problem into an MDP with a B\"uchi objective, which is a standard approach used by probabilistic model-checking~\cite{baier2008principles,sickert2016limitdeterministic} and other LTL planning works (e.g., ~\cite{bozkurt2020control}). 
Then, we show that the satisfaction probability of the B\"uchi objective can be expressed in another form, i.e. the value function of the surrogate reward.

\subsection{Modeling planning problems with LTL objectives as MDPs with B\"uchi objectives}

We propose our model called \textit{Markov decision processes with B\"uchi objective}. 
It augments general MDPs \cite{baier2008principles} with a set of accepting states.
\begin{definition} \label{def:mdp}
A Markov decision process with B\"uchi objective is a tuple $\mathcal{M} = (S, A, P, s_\mathrm{init}, B)$ where
\begin{itemize}
 \setlength{\itemsep}{0pt}
 \item $S$ is a finite set of states and $s_\mathrm{init} \in S$ is the initial state,
 \item $A$ is a finite set of actions 
 where $A(s)$ denotes the set of allowed actions in the state $s\in S$,
 \item $P: S \times A \times S \to [0,1]$ is the transition probability function such that for all $s\in S$, we have 
 \[
 \sum_{s'\in S}P(s,a,s') = \begin{cases}
  1, & a \in A(s) \\
  0, & a \notin A(s)
 \end{cases},
 \]
 \item $B\subseteq S$ is the set of accepting states. The set of rejecting states $\neg B := S\backslash B$. 
\end{itemize}
\end{definition}
A path of the MDP $\mathcal{M}$ is an infinite state sequence $\sigma = s_0 s_1 s_2 \cdots$ such that for all $i \ge 0$, there exists $a_i \in A(s)$ and $s_{i}, s_{i+1} \in S$ with $P(s_i,a_i,s_{i+1}) > 0$. 
Given a path $\sigma$, the $i$th state is denoted by $\sigma[i] = s_i$. We denote the prefix by $\sigma[{:}i] = s_0 s_1\cdots s_i$ and suffix by $\sigma[i{+}1{:}] = s_{i+1} s_{i+2}\cdots$. 
We say a path $\sigma$ satisfies the B\"uchi objective $\varphi_B$ if $\mathrm{inf}(\sigma)\cap B \ne\emptyset$. Here, $\mathrm{inf}(\sigma)$ denotes the set of states visited infinitely many times on $\sigma$.



Our model is valid as the LTL objective
can be translated into a B\"uchi objective. 
The translation is done by constructing a product MDP from the original MDP and a limit-deterministic B\"uchi automaton \cite{sickert2016limitdeterministic} generated from the LTL objective. 
Instead of looking for the optimal memory-dependent policy for the LTL objective, one can find the optimal memoryless policy on the product MDP~\cite{sickert2016limitdeterministic}. 
Since the maximum satisfaction probability of B\"uchi objective can be achieved by a memoryless deterministic policy.



\subsection{Policy evaluation for B\"uchi objective via reachability}
We change an LTL planning problem into seeking the policy that maximizes the satisfaction probability of the B\"uchi objective. 
Evaluation of the policy requires the calculation of the satisfaction probability, which can be done by calculating the reachability probability. 

\begin{definition} \label{def:policy} 
A memoryless policy $\pi$ is a function $\pi:S \to {A}$ such that $\pi(\sigma[n])\in {A}(\sigma[n])$.
Given an MDP $\mathcal{M} = ( S, A, P, s_{init},B)$ and a memoryless policy $\pi$, a Markov chain (MC) induced by policy $\pi$ is a tuple $\mathcal{M_\pi}=(S, P_\pi,s_{init},B)$ where $P_\pi(s,s')=P(s,\pi(s),s')$ for all $s,s'\in S$.
\end{definition}
Under the policy $\pi$, for a path $\sigma$ starting at state $s$, its satisfaction probability of the B\"uchi objective $\varphi_B$ is defined as
\begin{align}\label{eqn:SP}
 \mathbb{P}_\pi(s,B) := Pr_{\sigma \sim \mathcal{M}_\pi} \big(\mathrm{inf}(\sigma)\cap B \ne\emptyset \mid \exists t: \sigma[t]=s \big). 
\end{align}

The calculation of the satisfaction
probability can be done by calculating the reachability probability. 
The probability of a path under the policy $\pi$ satisfying the B\"uchi objective equals the probability of a path entering the accepting bottom strongly connected component (BSCC) of the induced MC $\mathcal{M}_\pi$. 
\begin{definition} \label{def:BSCC}
A bottom strongly connected component (BSCC) of an MC is a strongly connected component without outgoing transitions. 
A strongly connected component of an MC is a communicating class, which is a maximal set of states that communicate with each other. 
A BSCC is rejecting\footnote{Here we call a state $s\in B$ as an accepting state, a state $s\notin B$ as a rejecting state. Notice that an accepting state must not exist in a rejecting BSCC, and a rejecting state may exist in an accepting BSCC.} if all states $s \notin B$. Otherwise, we call it an accepting BSCC. 
\end{definition}
Any path on the MC will eventually enter a BSCC and stay there. Any path entering an accepting BSCC will visit accepting states infinitely many times, therefore satisfying the B\"uchi objective. 

The calculation of the reachability probability can be done via dynamic programming after one detects all the BSCCs. 
This approach is adopted by probabilistic model-checking and requires complete model knowledge. 

\subsection{Surrogate reward that approximates the satisfaction probability}
A surrogate reward allows one to calculate the satisfaction probability even without knowledge of BSCCs. 
This equivalent representation of the B\"uchi objective allows one to transfer an LTL objective as a discounted reward. 

We study the surrogate reward for B\"uchi objective proposed in \cite{bozkurt2020control}, which is also widely used in \cite{voloshin2023eventual,shao2023sample,cai2021modular,hasanbeig2023certified,cai2023safe}. 
This surrogate reward is designed in a way that its value function approximates the satisfaction probability.
It consists of a reward function $R: S\to\mathbb{R}$ and a state-dependent discount factor function $\Gamma: S\to (0,1]$ with two discount factors $0<\gamma_B<\gamma\le 1$,
\begin{align}\label{eqn:surrogate}
R (s):= 
\begin{cases}
1-\gamma_{B} & s\in B \\ 
0 & s\notin B 
\end{cases}
,\quad \Gamma (s):= 
\begin{cases}
\gamma_{B} & s\in B \\ 
\gamma & s\notin B 
\end{cases}.
\end{align}
A positive reward is collected only when an accepting state is visited along the path. 
For this surrogate reward, the $K$-step return ($K\in \mathbb{N}$ or $K=\infty$) of a path from time $t\in \mathbb{N}$ is 
\begin{align}
&G_{t:K}(\sigma) := \sum_{i=0}^{K} R(\sigma[t+i])\cdot \prod_{j=0}^{i-1} \Gamma(\sigma[t+j]) \notag \\
&G_{t}(\sigma) := \lim_{K\to\infty}G_{t:K}(\sigma).
\end{align}
Here, the definition follows a standard discounted reward setting \cite{sutton2018reinforcement} but allows state-dependent discounting. Suppose the discount factor $\gamma=1$. If a path satisfies the B\"uchi objective, its return shall be a summation of a geometric series as $\sum_{i=0}^\infty(1-\gamma_B)\gamma_B^i = \frac{1-\gamma_B}{1-\gamma_B} = 1$. 
 
The value function $V_\pi(s)$ is the expected return conditional on the path starting at $s$ under the policy $\pi$. And it is related to B\"uchi objective as follows, 
\begin{align}\label{eqn:value}
&V_\pi(s) = \mathbb{E}_\pi [G_t(\sigma)\,|\,\sigma[t] =s]\notag \\
&= \mathbb{E}_\pi[G_t(\sigma)\mid\sigma[t]=s, \mathrm{inf}(\sigma)\cap B \ne\emptyset] \cdot \mathbb{P}_\pi(s,B) \nonumber\\
    &+ \mathbb{E}_\pi[G_t(\sigma)\mid\sigma[t]=s, \mathrm{inf}(\sigma)\cap B =\emptyset] \cdot (1-\mathbb{P}_\pi(s,B)),  
\end{align}
where $1-\mathbb{P}_\pi(s,B)$ stands for the probability of a path not satisfying the B\"uchi objective conditional on the path starting at $s$. 
As $\gamma_B$, $\gamma$ close to $1$, the value function becomes close to $\pi(s,B)$ as
\begin{align}
    &\lim_{\gamma\to 1^-}\mathbb{E}_\pi[G_t(\sigma)\mid\sigma[t]=s, \mathrm{inf}(\sigma)\cap B \ne\emptyset] = 1 \notag \\
    &\lim_{\gamma_B\to 1^-}{\mathbb{E}_\pi[G_t(\sigma)\mid\sigma[t]=s, \mathrm{inf}(\sigma)\cap B =\emptyset]} = 0.
\end{align}

The setting $\gamma_B$ and $\gamma$ is critical for solving an discounted reward problem using the value iteration, or Q-learning. 
These methods are guaranteed to find the optimal policy for the surrogate reward when $\gamma<1$. 
To make sure the optimal policy for the surrogate reward is the optimal policy for the LTL objective, 
one has to take $\gamma$, $\gamma_B$ as close to $1$ as possible to reduce the error between the value function and the satisfaction probability. 
For $\gamma_B$, one can never take $\gamma_B=1$ as $1-\gamma_B=0$ can not serve as a positive reward. 
Setting $\gamma=1$ seems to work in several works. 
However, \cite{xuan2024uniqueness} exposes setting $\gamma=1$ would break the uniqueness of the solution of the Bellman equation, thus hindering a correct evaluation of the satisfaction probability.

\begin{remark}
Other surrogate rewards based on B\"uchi and Rabin automata have been studied but have flaws. 
The surrogate rewards~\cite{hasanbeig2019reinforcement} based on limit-deterministic B\"uchi automata assign a constant reward for ``good'' states with a constant discount factor. 
This approach is technically flawed, as demonstrated by \cite{hahn2020faithful}. 
Surrogate reward \cite{sadigh2014learning} based on Rabin automata assigns constant positive rewards to certain ``good'' states and negative rewards to ``bad'' states. 
However, this surrogate reward function is also not technically correct, as demonstrated in \cite{hahn2019omegaregular}. 
\end{remark}


\section{Problem Formulation and main result}
This section introduces a key challenge when one calculates the satisfaction probability and our answer to it. 
Specifically, one needs to utilize the recursive formulation of the Bellman equation to solve it. 
Thus, two conditions $\mathrm{(i)}$ the Bellman equation to have a unique solution, 
    and $\mathrm{(ii)}$ a discount works on every dynamic programming update are required. 
However, our surrogate reward breaks both conditions as it allows $\gamma = 1$. 
Here, we investigate how dynamic programming calculates the satisfaction probability under the state-dependent discounting.


\subsection{Bellman equation and sufficient condition for the uniqueness of the solution}  
Given a policy, the value function satisfies the Bellman equation.\footnote{We call $V_\pi(s)=R(s) + \Gamma(s) \sum_{s'\in S}P_\pi (s,s') V_\pi(s')$ as the Bellman equation and~$V_\pi^*(s)= \max_{a\in A(s)}\{R(s) + \Gamma(s) \sum_{s'\in S}P(s, a, s') V_\pi^*(s')\}$ as the Bellman optimality equation.} The Bellman equation is derived from the fact that the value of the current state is equal to the expectation of the current reward plus the discounted value of the next state. For the surrogate reward in the equation~\eqref{eqn:surrogate}, the Bellman equation is given as follows: 
\begin{align} \label{eqn:bellman}
V_\pi(s) 
&=
\begin{cases}
 1-\gamma_{B}+\gamma_{B}\sum_{s'\in S}{P_\pi(s,s')V_\pi(s')} & s\in B \\ 
\gamma\sum_{s'\in S}{P_\pi(s,s')V_\pi(s)} & s\notin B 
\end{cases}.
\end{align}


Previous work \cite{voloshin2023eventual,hasanbeig2023certified,shao2023sample} allows $\gamma=1$. 
However, setting $\gamma=1$ yields multiple solutions of the Bellman equations, raising concerns about applying dynamic programming or RL.
A sufficient condition to restrict the uniqueness of the solution is proposed in \cite{xuan2024uniqueness}. 
\begin{lemma}\label{thm:1}
The Bellman equation~\eqref{eqn:bellman} has the value function as the unique solution, 
if and only if the solution for any state in a rejecting BSCC is zero \cite{xuan2024uniqueness}. 
\end{lemma} 

\subsection{Open question on the convergence of dynamic programming}

Based on Lemma~\ref{thm:1}, dynamic programming is a way to compute the value function using the Bellman equation. Given an initialization of the approximate value function $U_{(0)}$ and the Bellman equation~\eqref{eqn:bellman}, we shall iteratively do the dynamic programming update as 
\begin{align}\label{eqn:dynamic programming backup}
U_{(k+1)} 
 =(1-\gamma_{B})
\begin{bmatrix}
\mathbb{I}_{m} \\ 
\mathbb{O}_{n} \\ 
\end{bmatrix} + \begin{bmatrix}
\gamma_{B}I_{m\times m} &   \\ 
  & \gamma I_{n\times n} \\ 
\end{bmatrix}P_{\pi} 
U_{(k)}, 
\end{align}  
where $m=\vert B\vert$ is the number of accepting states, 
$n=\vert \neg B \vert$ is the number of rejecting states. 
$\mathbb{I}$ and $\mathbb{O}$ are column vectors with all $1$ and $0$ elements, respectively. 
We expect the approximate value function $U_{(k)}$ to converge to the value function $V$ during updates. 

Suppose we initialize $U_{(0)}$ as a zero vector. The sufficient condition for the uniqueness of the solution holds for all $U_{(k)}$. The value function is the only fix-point for the above update. However, a convergence guarantee is still needed to show we can compute the satisfaction probability using the above dynamic programming update. 

When $\gamma < 1$, the convergence is guaranteed by the one-step contraction shown as 
\begin{align}\label{eqn:one-step}
    \Vert U_{(k+1)}-V \Vert_\infty \le  \gamma \Vert U_{(k)}-V \Vert_\infty. 
\end{align}
As $\gamma = 1$, this contraction no longer holds, and the convergence to the value function is still an open question.
This motivates us to study the following problem.
\begin{quote}
\textbf{Problem Formulation:} 
For an MDP with B\"uchi objective $\mathcal{M}$ by Definition~\ref{def:mdp} and the surrogate reward~\eqref{eqn:surrogate}, given a policy $\pi$. Starting with $U_{(0)}=\mathbb{O}$, show approximate value function $U_{(k)}$ updated by dynamic programming~\eqref{eqn:dynamic programming backup} will converge to the value function $V$ as $k\to \infty$. And give an upper bound for the error $\Vert U_{(k)}-V_\pi \Vert_\infty$. 
\end{quote}
In the following, we assume a fixed policy $\pi$, leading us to omit the $\pi$ subscript from most notation when its implication is clear from the context. 

\subsection{Overview on main result}
When $\gamma = 1$, we find a \textit{multi-step contraction} shown as
\begin{align}
    \Vert U_{(k+N)}-V \Vert_\infty \le  c \Vert U_{(k)}-V \Vert_\infty  
\end{align}
exists for the dynamic programming update, where $N\in \mathbb{N}^+$ and $c\in(0,1)$ is a constant. Even though rejecting states lacks discounting, their reachability to accepting states still provides contraction. With this finding, we claim the convergence guarantee as follows.

\begin{theorem}\label{main result}
    Given an MDP $\mathcal{M}$ and the surrogate reward~\eqref{eqn:surrogate}, given a policy $\pi$. Starting with $U_{(0)}=\mathbb{O}$, approximate value function $U_{(k)}$ updated by dynamic programming~\eqref{eqn:dynamic programming backup} will converge to the value function $V$ as
\begin{itemize}
\item if $\gamma<1$
    \begin{align}
{\left \Vert  U_{(k)}-V\right \Vert_\infty}&\le \gamma^k
{\left \Vert  V\right \Vert_\infty}, 
\end{align}
    \item     if $\gamma = 1$
    \begin{align}
{\left \Vert  U_{(k)}-V\right \Vert_\infty}&\le (1-(1-\gamma_B)\varepsilon^{n'} )^{\floor*{\frac{k}{n'+1}}}
{\left \Vert  V\right \Vert_\infty}, 
\end{align}
where $\varepsilon$ is the lower bound for all possible transitions, that is, for all $(s,s') \in \{(s,s') | P_\pi(s,s')>0\}, P_\pi(s,s')\ge \varepsilon >0$. 
\end{itemize}
\end{theorem}
As the theorem claims, one can use the surrogate reward and dynamic programming to compute the satisfaction probability of the B\"uchi objective on the MDP. Thus, the following corollary holds. 
\begin{corollary}
    For a product MDP constructed by the MDP and the limit-deterministic B\"uchi automaton \cite{sickert2016limitdeterministic} generated from the LTL objective. 
    Given a policy $\pi$ on the product MDPs. The surrogate reward and dynamic programming can be used to compute the satisfaction probability of the LTL objective on the MDPs. 
\end{corollary}

Here, we give the convergence guarantee based on the discount factor $\gamma, \gamma_B$ and the reachability to the discounted state expressed as $\varepsilon^{n'}$. The first bound is commonly seen for dynamic programming as each update provides one-step contraction when $\gamma<1$. The second bound relies on the multi-step contraction shown in the following section.

\section{Multi-step contraction and proof of the main result}
In this section, we will formally prove the main result by exploiting the reachability from undiscounted rejecting states to discounted accepting states. 
First, we simplify the dynamic programming update~\eqref{eqn:dynamic programming backup} using Lemma~\ref{thm:1}. 
Then, we show the infinite norm of the error vector will surely be contracted when the update is applied enough times.  
\subsection{Dynamic programming for states outside rejecting BSCCs}
The sufficient condition in Lemma~\ref{thm:1} always holds as we initialize $U_{(0)}=\mathbb{O}$. Since the approximate value function on the rejecting BSCCs stays at zero.  
We can simplify the dynamic programming update~\eqref{eqn:dynamic programming backup} by dropping all states inside rejecting BSCCs, 
\begin{align}\label{eqn:dynamic programming short}
&\begin{bmatrix}
U^{B} \\ U^{\neg B_{T,A}}\\ 
\end{bmatrix}_{(k+1)} =(1-\gamma_{B})
\begin{bmatrix}
\mathbb{I}_{m} \\ 
\mathbb{O}_{n'} \\ 
\end{bmatrix} +H
\begin{bmatrix}
U^{B} \\ U^{\neg B_{T,A}}\\ 
\end{bmatrix}_{(k)}, 
\end{align} 
where 
\begin{align}
 H &=  {\begin{bmatrix}
 \gamma_{B}I_{m\times m} & \\ 
 & \gamma I_{n'\times n'} \\ 
 \end{bmatrix}} 
 \underbrace{\begin{bmatrix}
 P_{B \rightarrow B} & P_{B\rightarrow \neg B_{T,A}}\\ 
 P_{\neg B_{T,A}\rightarrow B} & P_{\neg B_{T,A}\rightarrow \neg B_{T,A}}\\ 
\end{bmatrix}}_{T} \notag \\
&= \begin{bmatrix}
 \gamma_B P_{B \rightarrow B} & \gamma_B P_{B\rightarrow \neg B_{T,A}}\\ 
 \gamma P_{\neg B_{T,A}\rightarrow B} & \gamma P_{\neg B_{T,A}\rightarrow \neg B_{T,A}}\\ 
\end{bmatrix}.
\end{align}
Here, the state space is partitioned into three sets of states $B$, $\neg B_R$, $\neg B_{T,A}$ representing the set of accepting states, the set of states in the rejecting BSCCs and the set of remaining rejecting states. 
$U_\pi^{B}\in\mathbb{R}^{m}$, $U_\pi^{\neg B_{T,A}}\in\mathbb{R}^{n'}$ are the vectors listing the approximate value function for all $s\in {B}$, $s\in {\neg B_{T,A}}$, respectively. 
Matrix $T$ represents the transition between states in $X:=\{B, \neg B_{T,A}\}$. 
Sub-matrix $P_{B\rightarrow B}$, $P_{B\rightarrow \neg B_{T,A}}$ etc. contains the transition probability from a set of states to a set of states. 

At each dynamic programming update, the approximation error is updated by a linear mapping as 
\begin{align}\label{eqn:error}
D_{(k+1)} &= HD_{(k)}. 
\end{align}
where
\begin{align}
    {D}_{(k)}= \begin{bmatrix}
U^{B} \\ U^{\neg B_{T,A}}\\ 
\end{bmatrix}_{(k)}-\begin{bmatrix}
V^{B} \\ V^{\neg B_{T,A}}\\ 
\end{bmatrix} . 
\end{align}
$V_\pi^{B}\in\mathbb{R}^{m}$, $V_\pi^{\neg B_{T,A}}\in\mathbb{R}^{n'}$ are the vectors listing the value function for states in $B$ and $ {\neg B_{T,A}}$, respectively. 
When setting $\gamma = 1$, the requirement for the one-step contraction~\eqref{eqn:one-step}, which is $\Vert H \Vert_\infty<1$ does not hold. 
\subsection{Multi-step contraction}
Even with $\gamma = 1$, we can still show there exists a $N\in \mathbb{N}^+$ such that $\Vert H^N \Vert_\infty<1$. 
The multi-step update is given as
\begin{align}\label{multi-update}
    D_{(k+N)} = H^N D_{(k)} .
\end{align}
By showing each row sum of $H^N$ is strictly less than that of $T^N$, we guarantee $\Vert H^N \Vert_\infty<1$. 
In $T^{N}$, each element $\{T^{N}\}_{ij}$ is the probability of a path starting at $i$ visiting $j$ in the next $N$ steps, thus $\Vert T^N \Vert_\infty \le 1$.  
In $H^{N}$, each element $\{H^{N}\}_{ij}$ is expressed as a ``discounted'' version of $\{T^{N}\}_{ij}$,  
\begin{align}\label{eqn:dis prob}
 &\{H^{N}\}_{ij} = \sum_{s_1\in X}{P_{\pi,\gamma_B}(i,s_1)}\sum_{s_2\in X}{P_{\pi,\gamma_B}(s_1,s_2)} \notag \\ 
 &\cdots\sum_{s_{N-1}\in X}{P_{\pi,\gamma_B}(s_{N-2},s_{N-1})}{P_{\pi,\gamma_B}(s_{N-1},j)},
\end{align}
where 
\begin{align}\label{eqn: discounted}
P_{\pi,\gamma_B}(s,s')=
\begin{cases}
\gamma_B P_{\pi}(s,s')& s\in B \\ 
P_{\pi}(s,s') & s\in X\backslash B
\end{cases}.
\end{align}
Whenever a one-step transition starting from the accepting state happens, a discount $\gamma_B$ shall be applied in \eqref{eqn: discounted}, making $\{H^{N}\}_{ij}<\{T^{N}\}_{ij}$. 

The row sum of $H^N$ shall be strictly less than the corresponding row sum of $T^N$ if the probability of visiting an accepting state within the future $N-1$ steps is greater than zero. 
Thus, the contraction $\Vert H^N \Vert_\infty<1$ is brought in by the reachability from rejecting states to accepting, which is formalized as the following Lemma. 

\begin{lemma}\label{lem:leak}
 Starting with a state in $\neg B_{T,A}$, the probability of a path not visiting the set $B$ in $n':=\vert \neg B_{T,A}\vert$ steps is upper bounded by $1-\varepsilon^{n'}$. 
\end{lemma}
\begin{proof}
A transition leaving $\neg B_{T,A}$ must exist since all states in $\neg B_{T,A}$ are either 
\begin{enumerate}
 \item a recurrent state inside an accepting BSCC. Any path starting from such a recurrent state will eventually meet an accepting state. Thus, at least one path leaves the set $\neg B_{T,A}$.
 \item A transient state that will enter a BSCC. Any path starting from such a transient state will eventually enter an accepting BSCC or a rejecting BSCC. 
 Either way, the transient leaving the set $\neg B_{T,A}$ will happen. 
\end{enumerate}
As for all state $i\in {\neg B_{T,A}}$, there exists at least one path that will leave ${\neg B_{T,A}}$ in $n'$ steps, 
\begin{align}
 &1- \mathbb{P}(s_1, \cdots s_{n'} \in {\neg B_{T,A}}|s_0 = i) \ge \varepsilon^{n'} \notag \\
 \Rightarrow &\mathbb{P}(s_1, \cdots s_{n'} \in {\neg B_{T,A}}, s_{n'+1}\in S|s_0 = i) \le 1- \varepsilon^{n'}.
\end{align}
The existence of such a path can be shown by constructing a path starting at $i$ and never visiting any states in ${\neg B_{T,A}}$ more than once. One can use the diameter of a graph to prove this existence formally and we omit it for space considerations.
\end{proof}

By this reachability property, we show the multi-step contraction, which is technically described by the following lemma. 
\begin{lemma}\label{lem:main}
When $\gamma = 1$, given the approximation error $D_{(k)}$ updated by equation~\eqref{eqn:error} , there always exists a constant $c\in(0,1)$ and a positive integer $N\le \vert \neg B_{T,A}\vert+1$ such that 
 \begin{align}\label{eqn:multi-stage}
    \Vert D_{(k+N)} \Vert_\infty \le  c \Vert D_{(k)} \Vert_\infty. 
\end{align}
\end{lemma} 

\begin{proof}
For the first $m$ rows of $H^{N}$, every path starts at an accepting state and receives discounting in the beginning. For all $i\le m$, each element $\{H^{N}\}_{ij}\le \gamma_B\{T^{N}\}_{ij}$, we have $\sum_{j\in X} \{H^{N}\}_{ij} \le \gamma_B\sum_{j\in X} \{T^{N}\}_{ij}\le \gamma_B$. 

The remaining $n'$ rows need additional treatment, as discounting may not happen. 
However, we can rule out the case when no discount happens by setting $N = n'+1$. 
 
We split the sum of each row of $T^{n'+1}$ into two parts, 
\begin{align}\label{eqn:long}
 &\sum_{j\in X}\{T^{n'+1}\}_{ij} = \mathbb{P}(s_1,\cdots, s_{n'+1} \in X|s_0 = i) \notag \\
 &= \underbrace{\eta}_{\textrm{no accepting states are visited in } n' \textrm{ steps}} \notag \\ 
 &+\underbrace{\mathbb{P}(s_1,\cdots, s_{n'+1} \in X|s_0 = i)- \eta}_{\textrm{at least one accepting state is visited in } n' \textrm{ steps}} \notag \\
 &\le 1, 
\end{align}
where $\eta$ represents all the paths that will not visit accepting states in $n'$ steps and thus will not get discounted. 
Suppose $\eta =1$, the sum of the entire row of $H^{n'+1}$ shall be $1$. 
However Lemma~\ref{lem:leak} prohibits $\eta = 1$ from happening, 
\begin{align}
 \eta &= \mathbb{P}(s_1, \cdots s_{n'} \in {\neg B_{T,A}},s_{n'+1}\in X|s_0 = i)\notag \\ 
 &\le \mathbb{P}(s_1,\cdots, s_{n'}\in \neg B_{T,A}, s_{n'+1}\in S |s_0 = i) \notag \\
 &\le 1- \varepsilon^{n'}. 
\end{align}
The second part in~\eqref{eqn:long} represents all paths that will visit the accepting state in $n'$ steps at least once.
Thus, in $\sum_{j\in X}\{H^{n'+1}\}$ which represents sum of ``discounted" probability~\eqref{eqn:dis prob}, the first part stays the same. Meanwhile, the second part has to be discounted at least once, 
\begin{align}
 &\sum_{j\in X}\{H^{n'+1}\}_{ij} \notag \\
 &\le \underbrace{\eta}_{\textrm{no discounting}} \notag \\
 & + \gamma_B \underbrace{(\mathbb{P}(s_1,\cdots, s_{n'+1} \in X|s_0 = i)- \eta)}_{\textrm{at least one accepting state is visited in n' steps}} \notag \\
 &\le \eta + \gamma_B(1-\eta ) \notag \\
 &\le 1-(1-\gamma_B)\varepsilon^{n'} .
\end{align}
Thus,  the sum of each row of $H^{n'+1}\le c$ where $c=1-(1-\gamma_B)\varepsilon^{n'}$. And we have 
\begin{align}
    \Vert D_{(k+N)}  \Vert_\infty & =  \Vert H^N D_{(k)} \Vert_\infty  
     \le  c \Vert D_{(k)} \Vert_\infty . 
\end{align}
\end{proof}

The multi-step update always provides a multi-step contraction, so we can upper bound the convergence of approximation error. 
\subsection{Proof of Theorem~\ref{main result}} 
In the case of $\gamma = 1$, we get the $N$-step update of the approximation error as, 
\begin{align}
D_{(k+N)} &= H^{N}D_{(k)}. 
\end{align} 
By Lemma~\ref{lem:main}, after every $N$ update, the infinite norm of error must shrink by a constant $c<1$, 
\begin{align}
{\left \Vert D_{(nN)}\right \Vert_\infty}\le c^n
{\left \Vert D_{(0)}\right \Vert_\infty}. 
\end{align}
As $n\to \infty$, the error $\Vert {D}_{(nN)} \Vert_\infty\to 0$. 

Meanwhile, since the sum of each row of $H\le 1$, the error won't grow at each one-step update. Thus, we have the convergence as, 
\begin{align}
{\left \Vert D_{(k)}\right \Vert_\infty}&\le c^{\floor*{\frac{k}{N}}}
{\left \Vert D_{(0)}\right \Vert_\infty} .
\end{align}

Theorem~\ref{main result} also considers the case when $\gamma<1$. 
We get the upper bound for the approximation error as 
\begin{align}
{\left \Vert D_{(k)}\right \Vert_\infty}\le \gamma^k
{\left \Vert D_{(0)}\right \Vert_\infty}.
\end{align}
The upper bound on $\Vert D_{(k)} \Vert_\infty$ instantly holds for $\Vert U_{(k)} - V\Vert_\infty$ as starting the approximate value function at $\mathbb{O}$ guarantees $U^{\neg B_R}_{(k)}-V^{\neg B_R} = \mathbb{O}$, where $V^{\neg B_R}$ is the vector listing the value function for all states inside a rejecting BSCCs. 
{\hfill $\blacksquare$}

\section{Case Study}
In this section, we show our upper bound holds in a three-state Markov chain shown in Fig. \ref{fig:1}. 
$s_b$, $s_c$ are states in $\neg B_{T,A}$ and $s_a$ is the only accepting state.
All transitions are deterministic. The discount factor setting here is $\gamma_B=0.99$, $\gamma = 1$. 
The dynamic programming update is given as
\begin{align} 
U_{(k+1)} 
 =
\begin{bmatrix}
.01\\ 
0 \\ 
0\\
\end{bmatrix} + \begin{bmatrix}
.99 &  & \\ 
  & 1 &  \\ 
   & & 1  \\
\end{bmatrix}\begin{bmatrix}
0 &1&  0 \\ 
1 &0 &0  \\ 
0& 1& 0  \\
\end{bmatrix}
U_{(k)}. \notag
\end{align}  
The infinite norm of the $H$ matrix in~\eqref{multi-update} is shown as
\begin{align}
    \Vert H \Vert_\infty = \Vert H^2 \Vert_\infty = 1,\quad \Vert H^3 \Vert_\infty=0.99 \notag
\end{align}
The approximate value function is initialized as $U_{(0)}=\mathbb{O}$, thus the approximate error is $D_{(0)}=\mathbb{I}$ and during the first three updates we have, 
\begin{align} 
D_{(1)}
 =
\begin{bmatrix}
.99\\ 
1 \\ 
1\\
\end{bmatrix},\quad
D_{(2)}
 =
\begin{bmatrix}
.99\\ 
.99 \\ 
1\\
\end{bmatrix}
,\quad
D_{(3)}
 =
\begin{bmatrix}
.99^2\\ 
.99 \\ 
.99\\
\end{bmatrix} \notag .
\end{align}  
Since the discount $\gamma_B$ only works on $s_a$, $D(s_c)$ can decrease only after the $D(s_b)$ is decreased. It takes three updates for the $\Vert D \Vert_\infty$ to decrease from $1$ to $0.99$. 

The upper bound for estimation error provided by Theorem~\ref{main result} is ${\left \Vert U_{(k)}-V\right \Vert_\infty} \le (1-(1-\gamma_B)\varepsilon^{n'} )^{\floor*{\frac{k}{n'+1}}}$ where $\varepsilon = 1$ since all transitions are deterministic and $n'=2$ as $\vert \neg B_{T,A}\vert=2$.
Thus, our theorem yields a three-step contraction shown as 
\begin{align}
    {\left \Vert U_{(k)}-V\right \Vert_\infty} \le 0.99^{\floor*{\frac{k}{3}}} \notag
\end{align}
which captures the fact $\Vert H^3 \Vert_\infty=0.99$. 
The infinite norm of error shall be contracted by $\gamma_B$ after every three dynamic programming updates. 

In Fig.~\ref{fig:error2}, the error $\Vert D_{(k)}\Vert_\infty$ is decreased by $\gamma_B$ in the first three updates, aligning our upper bound. However, the error is decreased by $\gamma_B$ after every two updates when $k> 3$. The reason is that our bound on the infinite norm only considers the worst case. 
Where the two-step update 
\begin{align}
    H^2 = \begin{bmatrix}
.99 &0  & 0\\ 
0 & .99   & 0\\ 
1 & 0& 0\\
\end{bmatrix}\notag 
\end{align} 
decreases $D(s_a)$, $D(s_b)$ by $\gamma_B$ every two updates meanwhile making $D(s_c)_{(k+2)}=D(s_a)_{(k)}$. Since in vector $D_{(3)}$, $D(s_a)_{(3)}$ is the smallest element, thus $D_{(3)}$ serves as a good initialisation to make $\Vert D_{(k)}\Vert_\infty$ decreasing after every two updates.

\begin{figure}
 \centering
 \begin{tikzpicture}[>=stealth, node distance=3cm, on grid, auto]
 \node[state] (A) {$s_c$};
 \node[state, right=of A] (B) {$s_b$};
 \node[state, right=of B] (C) {$s_a$};

 \path[->]
 (A) edge[bend left] node {1} (B)
 (B) edge[bend left] node {1} (C)
 (C) edge[bend left] node {1} (B);
 \end{tikzpicture}
 \caption{Example of a three-state Markov Chain. A discount factor $\gamma_B<1$ and a reward $1-\gamma_B$ hold at $s_a$. Meanwhile, no rewards are gained at $s_b$ and $s_c$. Discount $\gamma$ is applied to $s_b$ and $s_c$ but $\gamma$ can be set to $1$. }\label{fig:1}
\end{figure}
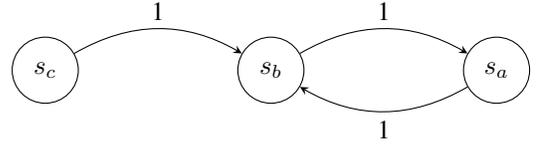

\begin{figure}
 \centering
 \includegraphics[width=0.5\textwidth]{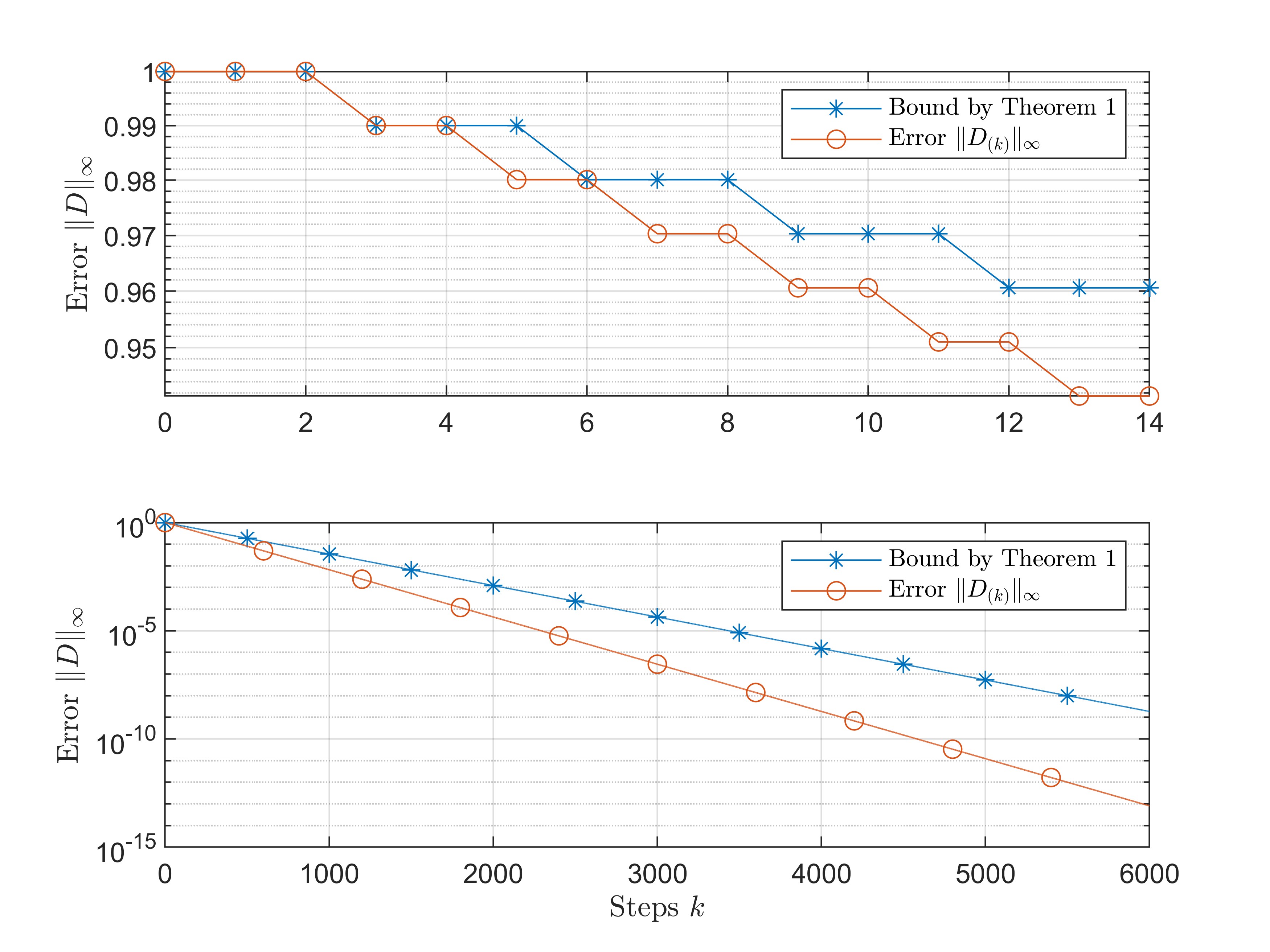}
 \caption{Exponential convergence of estimation error $\Vert D_{(k)} \Vert$ during dynamic programming updates. Our theorem yields an upper bound that shrinks $\gamma_B$ after every three dynamic programming updates. This upper bound captures the convergence in this example.}
 \label{fig:error2}
\end{figure}

\section{Conclusion and future work}
This work answers a challenge when using surrogate reward with two discount factors for complex LTL objectives. 
Specifically, we can always find the value function using dynamic programming even if discounting does not hold in many states. 
We discuss the convergence when using dynamic programming and show that a multi-step contraction exists as we do dynamic programming updates enough times. 
Our findings have implications for the correct policy evaluation of LTL objectives. 

Our future effect is to investigate if the convergence result still holds as we apply value iteration and Q-learning. The challenge is that once the policy is updated, it may induce a new MC, which may have different rejecting BSCCs. The sufficient condition for the uniqueness of the Bellman equation may not hold during policy updates.

\bibliography{references}
\end{document}